\documentclass[twoside]{article}
\usepackage[accepted]{aistats2eA}

\usepackage{amsfonts,mlapa,plain}
\usepackage{multirow}
\usepackage{latexsym}
\usepackage{epsfig}
\usepackage{bm}
\usepackage{times}
\usepackage{color}
\usepackage{algorithm, algorithmic}

\usepackage{amsmath, amsthm, amssymb}
\usepackage{graphicx}
\usepackage{enumitem}
\usepackage{thmtools}

\usepackage{graphs}

\long\def\DELETE#1{}

  \newtheorem{theorem}{Theorem}
  \newtheorem{lemma}[theorem]{Lemma}

	\newtheorem*{theorem*}{Theorem}
  \newtheorem*{lemma*}{Lemma}

  \theoremstyle{definition}

\declaretheoremstyle[notefont=\bfseries,notebraces={}{},%
    headpunct={},postheadspace=1em]{mystyle}
\def\L{\mathcal{L}}

\newcommand{\F}{{\cal F}}

\newcommand{\tn}{\textnormal}

\newcommand{\eps}{\epsilon}

\newcommand{\V}{{\cal V}}
\newcommand{\al}{\alpha_{ij}}
\newcommand{\N}{\textsf N}
\newcommand{\h}{\theta_i}
\newcommand{\W}{W_{ij}}
\newcommand{\x}{\xi_{ij}}
\newcommand{\s}{q_i}
\newcommand{\ce}{\mathcal{E}}
\newcommand{\A}{\cite{A} }
\newcommand{\M}{\mathcal{M}}
\newcommand{\m}{\gamma}

\newcommand{\bit}{\begin{itemize}}
\newcommand{\eit}{\end{itemize}}

\DeclareMathOperator*{\argmin}{arg\,min}

\begin{document}
\runningauthor{Weller, Jebara}
\twocolumn[
\aistatstitle{Approximating the Bethe partition function} 
\aistatsauthor{Adrian Weller \And Tony Jebara}

\aistatsaddress{ Columbia University, New York NY 10027 \\ \texttt{adrian@cs.columbia.edu}
 \And Columbia University,  New York NY 10027 \\ \texttt{jebara@cs.columbia.edu}} 
]

\begin{abstract}
When belief propagation (BP) converges, it does so to a stationary point of the Bethe free energy $\F$, and is often strikingly accurate. However, it may converge only to a local optimum or may not converge at all. An algorithm was recently introduced for attractive binary pairwise MRFs which is guaranteed to return an $\eps$-approximation to the global minimum of $\F$ in polynomial time provided the maximum degree $\Delta=O(\log n)$, where $n$ is the number of variables. Here we significantly improve this algorithm and derive several results including a new approach based on analyzing first derivatives of $\F$, which leads to performance that is typically far superior and yields a fully polynomial-time approximation scheme (FPTAS) for attractive models without any degree restriction. Further, the method applies to general (non-attractive) models, though with no polynomial time guarantee in this case, leading to the important result that approximating $\log$ of the Bethe partition function, $\log Z_B=-\min \F$, for a general model to additive $\epsilon$-accuracy may be reduced to a discrete MAP inference problem.  We explore an application to predicting equipment failure on an urban power network and demonstrate that the Bethe approximation can perform well even when BP fails to converge.

\end{abstract}

\section{INTRODUCTION}

Undirected graphical models, also termed Markov random fields (MRFs), are flexible tools used in many areas including speech recognition, systems biology and computer vision. 
A set of variables and a score function is specified such that the probability of a configuration of variables is proportional to the value of the score function, which typically factorizes into sub-functions over subsets of variables in a way that defines a topology on the variables. %This structure sometimes allows efficient computation of desired properties. 

\smallskip
Three central problems are:
\begin{enumerate}
\vspace{-.2cm} \item To evaluate the partition function $Z$, which is the sum of the score function over all possible settings, and hence is the normalization constant for the probability distribution.  
\vspace{-.2cm} \item Marginal inference, which is computing the probability distribution of a given subset of variables.  
\vspace{-.2cm} \item Maximum a posteriori (MAP) inference, which is the task of identifying a setting of all the variables which has maximum probability. 
\end{enumerate} 

\vspace{-.2cm}The first two problems are related (marginals are a ratio of two partition functions). Computing $Z$ belongs to the class of counting problems \#P \cite{Val79}. Further, exact marginal inference is NP-hard \cite{Cooper90}. The MAP problem is typically easier, yet is still NP-hard \cite{Shi94}, even 
to approximate \cite{AbdHed98}. Much work has focused on trying to find good approximate solutions, or restricted domains where exact solutions may be found efficiently. One popular method is to use a message-passing algorithm called belief propagation \cite{Pearl}, which returns an exact solution in linear time in $n$,  the number of variables, if the topology of the model is a tree. If this method is applied to general topologies, termed loopy belief propagation (LBP), results are sometimes strikingly good \cite{turbo98,Mur99}, though in general it may not converge at all, and if it does, it may not be to a global optimum.

\cite{YedFreWei01} showed a remarkable connection between LBP and an earlier approach from statistical physics \cite{Bethe35,Peierls36}, in that any fixed point of LBP corresponds to a stationary point of a function of the system, termed the Bethe free energy $\F$. In fact, LBP can be seen as an iteration of the fixed point equations of the Bethe free energy. Variational approaches led to a better understanding of this relationship, showing that the negative of the global minimum of the Bethe free energy is the $\log$ of the Bethe partition function $Z_B$. Thus, $Z_B$ should yield a good approximation to the true partition function $Z$, though this is not a formal result - there are cases where it performs poorly, typically when there are many short cycles with strong edge interactions \cite[\S~4.1]{WaiJor08}. Even then, however, it can still be remarkably effective and in practice, LBP is widely used, often with excellent results. One motivation for our algorithm is to allow exploration of the limits for when $Z_B$ performs well, even when LBP or other local optimization approaches fail, which has not previously been possible. We demonstrate this application in Experiments \S\ref{sec:exp}.

Another interesting example is the demonstration \cite{Chand11} that the Bethe
approximation is very useful to count independent sets of a graph. Further, it was shown that if the shortest cycle cover conjecture of Alon and Tarsi \cite{AloTar85} is true, then the Bethe approximation is very good indeed for a random 3-regular graph.

Extensive analysis has focused on understanding conditions under which LBP is guaranteed to converge to the global optimum \cite{Hes04,MK07,Wat11}, but outside these restricted settings, until recently, there were no polynomial time methods even to approximate $Z_B$.   
One major area of study is the important subclass of models which are \emph{binary}, i.e. each variable takes one of just two possible values, and \emph{pairwise}, i.e. all score sub-functions are evaluated over at most two variables. These play a key role in areas such as computer vision, both directly and as critical subroutines in solving more complex problems \cite{PleKoh12}. Further, it is possible to convert a general MRF into an equivalent binary pairwise model \cite{YedFreWei01}, though potentially with a much enlarged state space.

An algorithm was introduced in \cite{Shin12} guaranteed to return an approximately stationary point of $\F$ in polynomial time for such binary pairwise models, though with a bound on the maximum degree, $\Delta~=~O(\log n)$. \cite{A} then used a discretizing approach to derive a polynomial-time approximation scheme (PTAS) for $\log Z_B$ for the significant subclass of \emph{attractive}\footnote{An \emph{attractive} model has all pairwise relationships of the type that tend to pull adjacent variables toward the same value (see \S \ref{sec:prelims} for a more precise definition). Equivalent terms used are \emph{associative}, \emph{regular} or \emph{ferromagnetic}.} binary pairwise models, also with $\Delta~=~O(\log n)$.  Interestingly, \cite{Ruo12} recently proved that $Z_B \leq Z$ for attractive models. Similarly, for graphical models whose partition function is the permanent of a non-negative matrix, $Z_B$ is recoverable via convex optimization and, here too, $Z_B \leq Z$ \cite{HuaJeb09,vontobel2010bethe,watanabe2010belief,Gur11}. Otherwise, beyond trivial cases where the graph is acyclic, efficiently computing or approximating $Z_B$ remains an active research topic.

\subsection{Contribution and Summary}

We obtain important new results for binary pairwise MRFs as described in the Abstract. We adopt ideas from \cite{A} but go significantly further to derive much stronger results. The overall approach is to construct a \emph{sufficient mesh} of discretized points in such a way that the optimum mesh point $q^*$ is guaranteed to have $\F(q^*)$ within $\eps$ of the true optimum. The new, first derivative approach, generally results in a much coarser, yet still sufficient mesh, and also admits adaptive methods to focus points in regions where $\F$ may vary rapidly. Separately, we also refine the second derivative method of \cite{A} to derive a method that performs well for very small $\epsilon$. We then consider how best to solve the resulting discrete optimization problem, which may be framed as multi-label MAP inference, and for which many techniques are available, some of which are efficient for sub-classes of problem.\footnote{Computing $Z_B$ is at least PPAD or PLS-hard in general since it not only requires a fixed point but also the global minimizer \cite{Shin13,DasPap11}.} 

In \S \ref{sec:prelims}, we establish notation and present various preliminary results, then apply these in \S \ref{sec:1st} to present our new approach for mesh construction based on analyzing first derivatives of $\F$. This leads to much improved performance (often by orders of magnitude), immediately admits general (non-attractive) models, and in the attractive setting yields a FPTAS for models with no restriction on topology.

In \S \ref{sec:better2} we  revisit the second derivative approach of \cite{A}. We show how this method can be refined and extended to yield better performance and also to admit non-attractive models, though for most cases of interest, unless $\epsilon$ is very small, the method of \S \ref{sec:1st} will be superior.

In \S \ref{sec:MAP}, we discuss the derived discrete optimization problem, which may be viewed as a multi-label MAP inference problem. In certain settings the problem is tractable, and in general we mention several features that can make it easier to find a satisfactory solution, or at least to bound its value. Experiments are described in \S \ref{sec:exp} demonstrating practical application of the algorithm. Finally, we present conclusions in \S \ref{sec:conc}.

\subsubsection{Structure of the overall algorithm}\label{sec:overall}

\noindent Input: Parameters $\{\theta_i, W_{ij}\}$ for a general binary pairwise MRF (convert format using the reparameterization of \S \ref{sec:input} if required), and a desired accuracy $\epsilon$.
\begin{enumerate}%[nosep, label=\alph*), wide]%[(a)]
\setlength{\itemindent}{-1.0em}
\vspace{-.2cm} \item Preprocess by computing bounds $\{A_i,B_i\}$ on the locations of minima %using Theorem \ref{thm:qsand2} or any other method 
(see \S \ref{sec:prelimbounds}).
%\vspace{-.2cm} \item Given bounds $\{A_i,B_i\}$, compute $\{L_i, U_i\}$ parameters which improve bounds on partial derivatives of $\F$ (see \S *).
\vspace{-.2cm} \item Construct a sufficient mesh using one of the methods in this paper. Indeed, all approaches are fast, so several may be used, then the most efficient mesh selected.
\vspace{-.2cm} \item Attempt to solve the resulting multi-label MAP inference problem, see \S \ref{sec:MAP}.
\vspace{-.2cm} \item If unsuccessful, but a strongly persistent partial solution was obtained, then improved $\{A_i,B_i\}$ may be generated (see \S \ref{sec:pers}), repeat from 2.
\end{enumerate}
\vspace{-.2cm} At anytime, one may stop and compute bounds on $\F$, see \S \ref{sec:othercases}. 

\subsection{Related work}
Methods such as CCCP \cite{Yui02} or UPS \cite{TehWel02} are guaranteed to converge to a local minimum of the Bethe free energy, but this may be far from the global optimum.
In earlier work, a fully polynomial-time randomized approximation scheme (FPRAS) for the true partition function was derived \cite{JerSin93}, but only when singleton potentials are uniform (i.e. a uniform external field) and the resulting runtime is high at $O(\eps^{-2} m^3 n^{11} \log n)$.
It was recently shown \cite{HeiGlo11} that models exist such that the true marginal probability cannot possibly be the location of a minimum of the Bethe free energy.
Our work demonstrates an interesting connection between MAP inference techniques (NP-hard) and estimating the partition function $Z$ (\#P-hard). Recently \cite{HazJaa12} showed a different connection by using MAP inference on randomly perturbed models to approximate and bound $Z$.

\section{NOTATION \& PRELIMINARIES} \label{sec:prelims}

Our notation is similar to \A and \cite{WelTeh01}. We focus on a binary pairwise model with $n$ variables $X_1,\dots,X_n \in \mathbb{B}=\{0,1\}$ and graph topology $(\mathcal{V},\mathcal{E})$ with $m=|\ce|$; that is $\V$ contains nodes $\{1,\dots,n\}$ where $i$ corresponds to $X_i$, and $\ce \subseteq V \times V$ contains an edge for each pairwise score relationship. Let $\N(i)$ be the neighbors of $i$. Let $x=(x_1,\dots,x_n)$ be one particular configuration, and introduce the notion of \emph{energy} $E(x)$ through\footnote{The probability or score function can always be reparameterized in this way, with finite $\h$ and $W_{ij}$ terms provided $p(x)>0 \; \forall x$, which is a requirement for our approach. There are reasonable distributions where this does not hold, i.e. distributions where $\exists x: p(x)=0$, but this can often be handled by assigning such configurations a sufficiently small positive probability $\epsilon$.\label{fn:finite}}
\begin{equation}\label{eq:E}
p(x)=\frac{e^{-E(x)}}{Z}, \; E=-\sum_{i \in \mathcal{V}} \h x_i -\sum_{(i,j)\in \mathcal{E}} \W x_i x_j,
\end{equation}
where the partition function $Z=\sum_x e^{-E(x)}$ is the normalizing constant. 

Given any joint probability distribution $p(X_1,\ldots,X_n)$ over all variables, the %\emph{true} 
(Gibbs) free energy is defined as $\F_G(p)=\mathbb{E}_p(E)-S(p)$, where $S(p)$ is the (Shannon) entropy of the distribution. Using variational methods, a remarkable result is easily shown \cite{WaiJor08}: minimizing $\F_G$ over the set of all globally valid distributions (termed the \emph{marginal polytope}) yields a value of $-\log Z$, exactly at the true marginal distribution, given in \eqref{eq:E}. 

Minimizing $\F_G$ is, however, computationally intractable, hence the approach of minimizing the Bethe free energy $\F$ makes two approximations: (i) the marginal polytope is relaxed to the \emph{local polytope}, where we require only \emph{local} consistency, that is we deal with a \emph{pseudo-marginal} distribution $q$, which in our context may be considered $\{q_i=q(X_i=1) \; \forall i \in \V, \mu_{ij}=q(x_i,x_j) \; \forall (i,j)\in \ce\}$ subject to $q_i=\sum_j \mu_{ij} \; \forall i \in \V, j\in \N(i)$; and (ii) the entropy $S$ is approximated by the Bethe entropy $S_B=\sum_{(i,j)\in \ce} S_{ij} + \sum_{i \in \cal{V}} (1-d_i) S_i$, where $S_{ij}$ is the entropy of $\mu_{ij}$, $S_i$ is the entropy of the singleton distribution and $d_i=|\N(i)|$ is the degree of $i$. We assume the model is connected so $d_i \geq 1 \; \forall i$ (else each component may be analyzed independently), and take $x \log x = 0$ for $x=0$. Hence, the global optimum of the Bethe free energy,
\begin{align}\label{eq:defnF}
\F(q) &=\mathbb{E}_q(E)-S_B(q)\\
	 &= \sum_{(i,j)\in \ce} -\big(\W \x +S_{ij}(q_i,q_j) \big) \notag\\
	 &\qquad +\sum_{i \in \cal{V}} \big( -\h q_i + (z_i-1) S_i(q_i) \big), \notag
\end{align}
is achieved by minimizing $\F$ over the local polytope, with $Z_B$ defined s.t. the result obtained equals $-\log Z_B$. See \cite{WaiJor08} for details.

Considering the local polytope, given $q_i$ and $q_j$, we must have
\begin{equation}\label{eq:mu}
\mu_{ij}=\begin{pmatrix} 1 + \x -\s -q_j & q_j-\x \\ \s - \x & \x \end{pmatrix}
\end{equation}

for some $\x \in [0,\min(\s, q_j)]$, where $\mu_{ij}(a,b)=q(X_i=a,X_j=b)$. Let $\al=e^{\W}-1$. $\al=0 \Leftrightarrow \W=0$ may be assumed not to occur else the edge $(i,j)$ may be deleted. $\al$ has the same sign as $\W$, if positive then the edge $(i,j)$ is \emph{attractive}; if negative then the edge is \emph{repulsive}.
The MRF is attractive if all edges are attractive. As in \cite{WelTeh01}, one can solve for $\x$ explicitly in terms of $\s$ and $q_j$ by minimizing $\F$, leading to a quadratic equation with real roots,
\begin{equation}\label{eq:xi2}
\alpha_{ij}\xi_{ij}^2 - [1+\alpha_{ij}(q_i+q_j)]\xi_{ij}+(1+\alpha_{ij})q_i q_j=0.
\end{equation}
For $\al>0$, $\x(q_i,q_j)$ is the lower root, for $\al<0$ it is the higher. %When $\al=0$ (no edge relationship), this reduces as expected to $\x=q(X_i=1,X_j=1)=q(X_i=1)q(X_j=1)=q_i q_j$. 
Collecting the pairwise terms of $\F$ from \eqref{eq:defnF} for one edge, define
\begin{equation}\label{eq:f}
f_{ij}(q_i,q_j)=-\W \x(q_i, q_j) -S_{ij}(q_i,q_j).
\end{equation}
Thus we may consider the minimization of $\F$ over $q=(q_1,\dots,q_n) \in [0,1]^n$.

We are interested in \textit{discretized pseudo-marginals} where for each $q_i$, we restrict its possible values to a discrete mesh $\M_i$ of points in $[0,1]$, which may be spaced unevenly. We allow $\M_i \ne \M_j$. Write $\M$ for the entire mesh. Let $N_i=|\M_i|$ and define $N=\sum_{i \in \V} N_i$ and $\Pi=\prod_{i \in \V} N_i$, the sum and product respectively of the number of mesh points in each dimension. Let $\hat{q}$ be the location of a global optimum of $\F$. We say that a mesh construction $\M(\eps)$ is \emph{sufficient} if, given $\epsilon>0$, it can be guaranteed that $\exists$ a mesh point $q^* \in \prod_{i \in \V} \M_i$ s.t. $\F(q^*)-\F(\hat{q}) \leq \epsilon$. 

We shall make use of the standard sigmoid function, $\sigma(x)=1/(1+\exp(-x))$ for various bounds. 

\subsection{Input model specification}\label{sec:input}
Throughout this paper, we assume the reparameterization in \eqref{eq:E} for all analysis, but a different specification is more natural for input models avoiding bias. We assume an input model is given with singleton terms $\theta_i$ as in \eqref{eq:E}, but with pairwise energy terms instead given by $-\frac{W_{ij}}{2}x_i x_j -\frac{W_{ij}}{2} (1-x_i)(1-x_j)$. With this format, varying $W_{ij}$ simply alters the degree of push/pull between $i$ and $j$, without also changing the probability that each variable will be 0 or 1, as is the case with the format of \eqref{eq:E}.
We assume maximum possible values $W$ and $T$ are known with $|\h| \leq T \; \forall i \in \V$, and $|W_{ij}| \leq W \; \forall (i,j) \in \ce$. The required transformation to convert from input model to the format of \eqref{eq:E}, simply takes $\h \leftarrow \h - \sum_{j \in \N(i)} W_{ij}/2$, leaving $W_{ij}$ unaffected.

%We assume maximum possible values $W$ and $T$ are known with $|\h| \leq T \; \forall i \in \V$, and $|W_{ij}| \leq W \; \forall (i,j) \in \ce$.

\subsection{Submodularity}\label{sec:submodularity}
%Submodular functions, defined on a discrete lattice, are in some ways like convex and in some ways like concave continuous functions. 
In our context, a pairwise multi-label function on a set of ordered labels $X_{ij}=\{1,\dots,K_i\} \times \{1,\dots,K_j\}$ is \textit{submodular} iff $\forall x, y \in X_{ij}, \; f(x \wedge y) + f(x \vee y) \leq f(x) + f(y)$,  
where for $x=(x_1,x_2)$ and $y=(y_1,y_2)$, $(x \wedge y)=(\min(x_1,y_1),\min(x_2,y_2))$ and $(x \vee y)=(\max(x_1,y_1),\max(x_2,y_2))$. For binary variables, submodular energy is equivalent to being attractive. %See [*refs] for details.

The key property for us is that if all pairwise cost functions $f_{ij}$ over $\M_i \times \M_j$ from \eqref{eq:f} are submodular, 
then the global discretized optimum may be found efficiently %as a multi-label MAP inference problem 
using graph cuts \cite{SchFla06}.

\begin{theorem}[Submodularity for any discretization of an attractive model, \A Theorem 8, \cite{Kor12}]\label{thm:submod}
If a binary pairwise MRF is submodular on an edge $(i,j)$, i.e. $W_{ij}>0$, then the multi-label discretized MRF for any mesh $\M$ is submodular for that edge. In particular, if the MRF is fully attractive, i.e. $W_{ij}>0 \; \forall (i,j) \in \cal{E}$, then the multi-label discretized MRF is fully submodular for any discretization. Proof in \A.
\end{theorem}

\subsection{Flipping variables}\label{sec:flip}

As in \A, we use the techniques below for flipping variables, i.e. we can consider a new model with variables $\{X_i'\}$, where $X_i'=1-X_i$ for some selection of $i$. Flipping a variable flips the parity of all its incident edges so attractive $\leftrightarrow$ repulsive. Flipping both ends of an edge leaves its parity unchanged. 

\subsubsection{Flipping all variables}\label{sec:flipall}%*Eliminate this since covered as a special case of flipping some?

Consider a new model with variables $\{X_i'=1-X_i, i=1,\dots,n\}$ and the same edges. Instead of $\theta_i$ and $W_{ij}$ parameters, let those of the new model be $\theta_i'$ and $W_{ij}'$. Identify values such that the energies of all states are maintained up to a constant\footnote{Any constant difference will be absorbed into the partition function and leave probabilities unchanged.}:
\begin{align*}
E &= -\sum_{i \in \cal{V}} \theta_i X_i - \sum_{(i,j)\in \ce} W_{ij} X_i X_j \\
&= const -\sum_{i \in \cal{V}} \theta_i' (1-X_i) - \!\! \sum_{(i,j) \in \ce} W_{ij}'(1-X_i)(1-X_j).
\end{align*}
Matching coefficients gives
\begin{equation}\label{eq:flipall}
W_{ij}'=W_{ij}, \; \theta_i'=-\theta_i-\sum_{j \in \N(i)} W_{ij}.%=-\theta_i-W_i.
\end{equation}
If the original model was attractive, so too is the new.

\subsubsection{Flipping some variables}\label{sec:flipsome}

Sometimes it is helpful to flip only a subset $\cal{R} \subseteq \cal{V}$ of the variables. This can be useful, for example, to make the model locally attractive around a variable, which can always be achieved by flipping just those neighbors to which it has a repulsive edge. Let $X_i' = 1-X_i$ if $i \in \cal{R},$ else $X_i'=X_i$ for $i \in \cal{S}$, where $\cal{S}=\cal{V} \setminus \cal{R}$. Let $\mathcal{E}_t=\{$edges with exactly $t$ ends in $\cal{R}\}$ for $t=0,1,2$. 

As in \ref{sec:flipall}, solving for $W_{ij}'$ and $\theta_i'$ such that energies are unchanged up to a constant,
\begin{align}
\label{eq:flipsome}
W_{ij}' &= \begin{cases} W_{ij} & \mspace{-1mu}(i,j) \in \mathcal{E}_0 \cup \mathcal{E}_2,\\
-W_{ij} & \mspace{-1mu} (i,j) \in \mathcal{E}_1
\end{cases} \notag \\ 
%\mspace{-4mu}
\theta_i' &= \begin{cases} \h + \sum_{(i,j) \in \mathcal{E}_1} W_{ij} & \mspace{-1mu} i \in \cal{S}, \\ -\h - \sum_{(i,j) \in %\mathcal{E}_1 \cup 
\mathcal{E}_2} W_{ij} & \mspace{-4mu} i \in \cal{R}. \end{cases}
\end{align}

\begin{lemma}\label{lem:flipBethe}
Flipping variables changes affected pseudo-marginal matrix entries' locations but not values. $\F$ is unchanged up to a constant, hence the locations of stationary points are unaffected. (Proof in \cite{A})
\end{lemma}

\subsection{Preliminary bounds}\label{sec:prelimbounds}

We use the following results from  \cite{A}.

\begin{lemma}[\A Lemma 2]\label{lem:newst}
$\al \geq 0 \Rightarrow \x \geq q_i q_j, \al \leq 0 \Rightarrow \x \leq q_i q_j$
\end{lemma}

\begin{theorem}[\A Theorem 4]\label{thm:qsand2}
For general edge types (associative or repulsive), let $W_i=\sum_{j \in \N(i): W_{ij}>0} W_{ij}$, $V_i=-\sum_{j \in \N(i): W_{ij}<0} W_{ij}$. At any stationary point of the Bethe free energy, $\sigma(\theta_i-V_i) \leq q_i \leq \sigma(\theta_i+W_i)$. 
\end{theorem}

For the efficiency of our overall approach, it is very desirable to tighten the bounds on locations of minima of $\F$ since this both reduces the search space and allows a lower density of discretizing points in our mesh. %Further, our new approach in \S\ref{sec:1st}  
This may be achieved efficiently by running either of the following two algorithms: Bethe bound propagation (BBP) from \cite{A}, or using the approach from \cite{MK07} which we term MK. Either method can achieve striking results quickly, though MK is our preferred method\footnote{Both BBP and MK are anytime methods that converge quickly, and can be implemented such that each iteration runs in $O(m)$ time. MK takes a little longer but can yield tighter bounds.} - it considers cavity fields around each variable and determines the range of possible beliefs after iterating LBP, starting from any initial values; since any minimum of $\F$ corresponds to a fixed point of LBP \cite{YedFreWei01}, this bounds all minima.

Let the lower bounds obtained for $q_i$ and $1-q_i$ respectively be $A_i$ and $B_i$
so that $A_i \leq q_i \leq 1-B_i$, and let the \emph{Bethe box} be the orthotope 
given by $\prod_{i \in \V} [A_i, 1-B_i]$.
Define $\eta_i=\min(A_i,B_i)$, i.e. the closest that $q_i$ can come to the extreme values of $0$ or $1$.

\begin{lemma}[Upper bound for $\x$ for an attractive edge, \A Lemma 6]\label{lem:xiub}
If $\al>0$, then 
$\x-\s q_j \leq \frac{\al m (1-M)}{1+\al}$, where $m=\min(q_i,q_j)$ and $M=\max(q_i,q_j)$.
\end{lemma}

\subsection{Derivatives of $\F$}
In \cite{WelTeh01}, first partial derivatives of the Bethe free energy are derived as
\begin{align}\label{eq:1stderiv}
\frac{\partial \F}{\partial q_i} &= -\theta_i + \log Q_i, \\
\text{where }
Q_i &= \frac{ (1-q_i)^{d_i-1} }  { q_i^{d_i-1} } 
\frac{\prod_{j \in \N(i)} (q_i-\xi_{ij})}{\prod_{j \in \N(i)} (1+\xi_{ij}-q_i-q_j)}.\notag%$$
\end{align}

\begin{theorem}[Second derivatives for each edge, \A Theorem 7]\label{thm:2nderiv}
For any edge $(i,j)$, for any $\al$,  
\begin{equation*}
\frac{\partial^2f_{ij}}{\partial q_i^2} = \frac{1}{T_{ij}} q_j(1-q_j) , \quad 
\frac{\partial^2f_{ij}}{\partial q_j^2} = \frac{1}{T_{ij}} q_i(1-q_i)
\end{equation*}
\begin{align}\label{eq:Tij}
\frac{\partial^2f_{ij}}{\partial q_i \partial q_j} &= \frac{\partial^2f_{ij}}{\partial q_j \partial q_i} = \frac{1}{T_{ij}} (\s q_j - \x), \notag\\
\text{where } T_{ij} &= \s q_j(1-\s)(1-q_j) -(\x-\s q_j)^2\\
										&\geq 0 \text{ with equality iff } q_i \text{ or } q_j \in \{0,1\} \notag. 
\end{align}
%Further $\mu_{01}\mu_{10}-\mu_{00}\mu_{11} = \s q_j - \x$ and has the sign of $-\al$. %\\Proof in Supplement.
\end{theorem}

Incorporating all singleton terms gives the following result.
\begin{theorem}[All terms of the Hessian, see \A \S 4.3 and Lemma 9]\label{thm:H}
Let $H$ be the Hessian of $\F$ for a binary pairwise model, i.e. $H_{ij} = \frac{\partial^2 \F}{\partial q_i \partial q_j}$, and $d_i$ be the degree of variable $X_i$, then
\begin{align*}
H_{ii} &= - \frac{d_i-1}{q_i (1-q_i)} + \sum_{j \in \N(i)} \frac{q_j(1-q_j)}{T_{ij}} \geq \frac{1}{q_i(1-q_i)},\\ 
H_{ij} &= \begin{cases}
	\frac{\s q_j-\x}{T_{ij}} \quad &(i,j) \in \ce \\
	0 &(i,j) \notin \ce, i\neq j.
	\end{cases}
\end{align*}
\end{theorem}

\section{NEW APPROACH}\label{sec:1st}% FOR MESH CONSTRUCTION CONSIDERING FIRST DERIVATIVES}\label{sec:1st}

We develop a new approach to constructing a sufficient mesh $\M$ by analyzing bounds on the first derivatives of $\F$. This yields
 several attractive features:
\begin{itemize}
\setlength{\itemindent}{-1.0em}
\vspace{-.0cm} \item For attractive models, we obtain a FPTAS with worst case runtime $O(\epsilon^{-3} n^3 m^3 W^3)$ and no restriction on topology, as was required in \cite{A}.
\vspace{-.0cm} \item Our sufficient mesh is typically dramatically coarser than the earlier method of \cite{A}, leading to a much simpler subsequent MAP problem unless $\epsilon$ is very small.  Here, the sum of the number of discretizing points in each dimension, $N=O\left(\frac{n m W} {\epsilon} \right)$. For comparison, the earlier method, even after our improvements in \S\ref{sec:better2}, forms a mesh with\\ $N=O \left(\eps^{-1/2} n^{7/4} \Delta^{3/4} \exp \left[ \frac{1}{2}( W(1+ \Delta/2) + T) \right] \right)$. As an example, for the model in the experiments of \S\ref{sec:exp}, our new approach with the adaptive minsum method (see \S\ref{sec:adaptive}), yields a mesh with $N$ that is 8 orders of magnitude smaller than the earlier method.
\vspace{-.0cm} \item Our approach immediately handles a general model with both attractive and repulsive edges. Hence approximating $\log Z_B$ may be reduced to a discrete multi-label MAP inference problem. This is valuable due to the availability of many MAP techniques. We discuss this in \S \ref{sec:MAP}, where we consider when the MAP problem is tractable and examine approaches which may be tried in general.

\end{itemize}

First assume we have a model which is fully attractive around variable $X_i$, i.e. $W_{ij}>0 \; \forall j \in \N(i)$. From \eqref{eq:1stderiv} and Lemma \ref{lem:newst}, we obtain
\begin{equation}\label{eq:upbound}
\frac{\partial \F}{\partial q_i} = -\theta_i + \log Q_i
 \leq -\theta_i + \log \frac{q_i}{1-q_i}.
\end{equation}
Flip all variables (see \S \ref{sec:flipall}). Write $ ^\prime$ for the parameters of the new flipped model, which is also fully attractive, then using \eqref{eq:flipall} and \eqref{eq:upbound}, 
\begin{align*}
\frac{\partial \F'}{\partial q_i'} 
 &\leq -\theta_i' + \log \frac{q_i'}{1-q_i'}\\
% \Leftrightarrow -\frac{\partial \F}{\partial q_i} &\leq -(-\theta_i-W_i) + \log \frac{1-q_i}{q_i}\\
\Leftrightarrow -\theta_i-W_i + \log \frac{q_i}{1-q_i} &\leq \frac{\partial \F}{\partial q_i}.
\end{align*}

Combining this with \eqref{eq:upbound} yields the sandwich result
$$-\theta_i-W_i + \log \frac{q_i}{1-q_i} \leq \frac{\partial \F}{\partial q_i} \leq -\theta_i + \log \frac{q_i}{1-q_i}.$$

Now generalize to consider the case that $i$ has some neighbors $\cal{R}$ to which it is adjacent by repulsive edges. In this case, flip those nodes $\cal{R}$ (see \S \ref{sec:flipsome}) to yield a model, which we denote by $ ^{\prime\prime}$, which is fully attractive around $i$, hence we may apply the above result. By \eqref{eq:flipsome} we have $\theta_i''=\theta_i-V_i$, and using $W_i''=W_i+V_i$, we obtain that for a general model,
\begin{equation}\label{eq:bounds}
-\theta_i-W_i + \log \frac{q_i}{1-q_i} \leq \frac{\partial \F}{\partial q_i} \leq -\theta_i + V_i + \log \frac{q_i}{1-q_i}.
\end{equation}

This bounds each first derivative $\frac{\partial \F}{\partial q_i}$ within a range of width $V_i+W_i=\sum_{j \in \N(i)} |W_{ij}|$, which will be sufficient for the main theoretical result to come in \eqref{eq:Nsimple}. We take the opportunity, however, to narrow this range, thereby improving the result in practice, by using just one step of the belief propagation algorithm (BBP) of \cite{A}.

Following the derivation of BBP in the Supplement of \cite{A}, where better bounds are derived on the $q_i$ location of stationary points by taking account of $[A_j,1-B_j]$ bounds on neighbors $j \in \N(i)$, we may refine the result of \eqref{eq:bounds} to yield 
\begin{align}\label{eq:LUbounds}
 &f_i^L(q_i) \leq \frac{\partial \F}{\partial q_i} \leq f_i^U(q_i), \text{ where} \nonumber \\
f_i^L(q_i) &=-\theta_i-W_i + \log U_i + \log \frac{q_i}{1-q_i} \nonumber \\
f_i^U(q_i) &=
-\theta_i + V_i -\log L_i + \log \frac{q_i}{1-q_i}.
\end{align}
$L_i, U_i$ are each $>1$ with $\log L_i + \log U_i \leq V_i + W_i$. They are computed as $L_i = \prod_{j \in \N(i)} L_{ij}$, $U_i=\prod_{j \in \N(i)} U_{ij}$, with 
$L_{ij} = \begin{cases}
1+ \frac{\al A_j}{1+\al (1-B_i) (1-A_j)} & \text{if } W_{ij}>0 \\
1+ \frac{\al B_j} {1+\al (1-B_i) (1-B_j)} & \text{if } W_{ij}<0 \end{cases}
$ , \\
$U_{ij} = \begin{cases}
1+ \frac{\al B_j}{1+\al (1-A_i) (1-B_j)} & \text{if } W_{ij}>0 \\
1+ \frac{\al A_j} {1+\al (1-A_i) (1-A_j)} & \text{if } W_{ij}<0 \end{cases}
$.

See Figure \ref{fig:partial} for an example. We make the following observations:
\begin{itemize}
\setlength{\itemindent}{-1.0em}
\vspace{-.3cm} \item The upper bound is equal to the lower bound plus the constant $D_i=V_i+W_i-\log L_i -\log U_i \geq 0$.
\vspace{-.3cm} \item The bound curves are monotonically increasing with $q_i$, ranging from $-\infty$ to $+\infty$ as $q_i$ ranges from $0$ to $1$.
\vspace{-.3cm} \item A necessary condition to be within the Bethe box is that the upper bound is $\geq 0$ and the lower bound is $\leq 0$. % (otherwise the derivative could not equal 0). 
Hence, anywhere within the Bethe box, we must have bounded derivative, $|\frac{\partial \F}{\partial q_i}| \leq D_i$. BBP generates $\{[A_i,1-B_i]\}$ bounds by iteratively updating with $L_i, U_i$ terms. In general, however, we may have better bounds from any other method, such as MK, which lead to higher $L_i$ and $U_i$ parameters and lower $D_i$.
\end{itemize}

\begin{figure}%[h] 
     %\centering   
     \includegraphics[width=9cm]{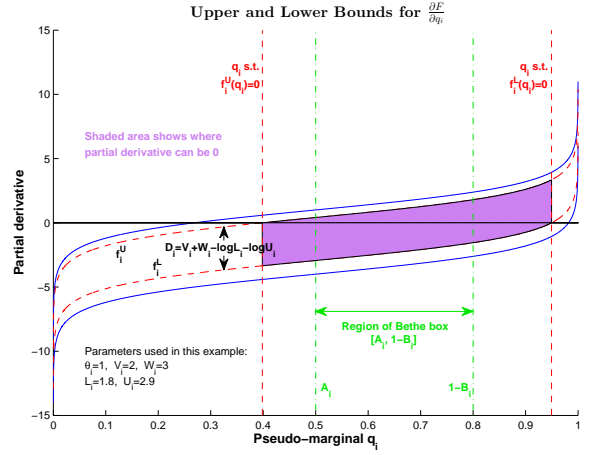} %prev using width 18
		 \caption{Upper and Lower Bounds for $\frac{\partial \F}{\partial q_i}$. Solid blue curves show worst case bounds \eqref{eq:bounds} as functions of $q_i$, and are different by a constant $V_i+W_i=\sum_{j \in \N(i)} |W_{ij}|$. Dashed red curves show the upper $f_i^U(q_i)$ and lower $f_i^L(q_i)$ bounds \eqref{eq:LUbounds} after being lowered by $\log L_i$ and raised by $\log U_i$ respectively, which incorporate the information from the bounds of neighboring variables. All bounding curves are strictly monotonic. The Bethe box region for $q_i$ must lie within the shaded region demarcated by vertical red dashed lines, but we may have better bounds available, e.g. from MK, as shown by $A_i$ and $1-B_i$.}
		 \label{fig:partial}
\end{figure}

$\F$ is continuous on $[0,1]^n$ and differentiable everywhere in $(0,1)^n$ with partial derivatives satisfying \eqref{eq:LUbounds}. $f_i^L(q_i)$ and $f_i^U(q_i)$ are continuous and integrable. Indeed, using the notation $\big[ \phi(x) \big]_{x=a}^{x=b} = \phi(b)-\phi(a)$,
\begin{equation}\label{eq:integral}
\int_a^b C + \log \frac{\s}{1-\s} dq_i = \Big[ C\s + \s \log \s + (1-\s) \log (1-\s) \Big]_{\s=a}^{\s=b}
\end{equation}
for $0 \leq a \leq b \leq 1$,
which relates to the binary entropy function $H(p)=-p \log p -(1-p) \log (1-p)$, recall the definition of $\F$. We remark that although $\frac{\partial \F}{\partial q_i}$ tends to $- \infty$ or $+ \infty$ as $q_i$ tends to $0$ or $1$, the integral converges (taking $0 \log 0 = 0$). 

Hence if $\hat{q}=(\hat{q}_1, \dots , \hat{q}_n)$ is the location of a global minimum, then for any $q=(q_1, \dots, q_n)$ in the Bethe box, 
\begin{equation}\label{eq:intbound}
\F(q)-\F(\hat{q}) \leq \!\! \sum_{i:\hat{q}_i \leq q_i} \int_{\hat{q}_i}^{q_i} \! f_i^U(q_i) dq_i + \!\! \sum_{i:q_i <\hat{q}_i} \int_{q_i}^{\hat{q}_i} \!\!\! -f_i^L(q_i) dq_i.
\end{equation}

To construct a sufficient mesh, 
a simple initial bound relies on $|\frac{\partial \F}{\partial q_i}| \leq D_i$. If mesh points $\M_i$ are chosen s.t. in dimension $i$ there must be a point $q^*$ within $\gamma_i$ of a global minimum (which can be achieved using a mesh width in each dimension of $2 \gamma_i$), then by setting  $\gamma_i=\frac{\epsilon}{nD_i}$, we obtain $\F(q^*)-\F(\hat{q}) \leq \sum_i D_i \frac{\epsilon}{nD_i} = \epsilon$. 
It is easily seen that $N_i \leq
1+\lceil \frac{1}{2 \gamma_i} \rceil$, hence the total number of mesh points, $N=\sum_{i \in \V} N_i$, satisfies
\begin{align}\label{eq:Nsimple}
N &\leq 2n+\frac{n}{2\epsilon} \sum_i D_i \leq 2n+\frac{n}{\epsilon} \sum_{(i,j) \in \ce} |W_{ij}| \notag \\
	&= O \left( \frac{n}{\epsilon} \sum_{(i,j) \in \ce} |W_{ij}| \right) = O \left( \frac{nmW}{\epsilon} \right),
\end{align}
 since $D_i \leq V_i+W_i = \sum_{j \in \N(i)} |W_{ij}|$. Here $W=\max_{(i,j) \in \ce}{|W_{ij}|}$ %, and $\sum_i d_i=2m$, where 
and $m=|\ce|$ is the number of edges.

If the initial model is fully attractive, then by Theorem \ref{thm:submod} we obtain a submodular multi-label MAP problem which is solvable using graph cuts with worst case runtime $O(N^3) = O(\eps^{-3} n^3 m^3 W^3)$ \cite{SchFla06,GrePorSeh89,Gol88}.

Note from the first expression in \eqref{eq:Nsimple} that if we have information on individual edge weights then we have a better bound using $\sum_{(i,j) \in \ce} |W_{ij}|$ rather than just  $mW$. %\cite{Gol88,Boy01,BoyKol04,SchFla06,RKAT08}. 
%, and in \S \ref{sec:discuss} provide numerical results indicating relative performance of different approaches in various contexts. 

For comparison, the earlier second derivative approach of \A has runtime $O(\epsilon^{-\frac{3}{2}} n^6 \Sigma^{\frac{3}{4}} \Omega^{\frac{3}{2}} )$, where, even using the improved method in \S \ref{sec:better2} here, $\Omega=O(\Delta e^{W(1+\Delta/2)+T})$. Unless $\epsilon$ is very small, the new first derivative approach is typically dramatically more efficient and more useful in practice. Further, it naturally handles both attractive and repulsive edge weights in the same way.
%[* High W should be easy?]

\subsection{Refinements, adaptive methods}\label{sec:refine} Since the resulting multi-label MAP inference problem is NP-hard in general \cite{Shi94}, it is  helpful to minimize its size.
As noted above, setting $\gamma_i=\frac{\epsilon}{n D_i}$, which we term the \emph{simple method}, yields a sufficient mesh, where $|\frac{\partial \F}{\partial q_i}| \leq D_i=V_i+W_i-\log L_i -\log U_i $. However, since the bounding curves are monotonic with $f_i^U \geq 0$ and $f_i^L \leq 0$, a better bound for the magnitude of the derivative is often available by setting $D_i = \max\{ f_i^U(1-B_i), -f_i^L(A_i) \}$.% (this can be no worse and will be better provided $[A_i,1-B_i]$ lies strictly within the $q_i$ points where $f_i^U(q_i)=0$ and $f_i^L(q_i)=0$).

\subsubsection{The minsum method}
We define $N_i=$ the number of mesh points in dimension $i$, with sum $N=\sum_{i \in \V} N_i$ and product $\Pi=\prod_{i \in \V} N_i$. For a fully attractive model, the resulting MAP problem may be solved in time $O(N^3)$ by graph cuts (Theorem \ref{thm:submod}, \cite{SchFla06,GrePorSeh89,Gol88}), so it is sensible to minimize $N$. In other cases, however, it is less clear what to minimize. For example, a brute force search over all points would take time $\Theta(\Pi)$.

Define the spread of possible values in dimension $i$ as $S_i=1-B_i-A_i$ and note $N_i = 1+ \lceil \frac{S_i}{2 \gamma_i} \rceil$ is required to cover the whole range. To minimize $N$ while ensuring the mesh is sufficient, consider the Lagrangian $\L= \sum_{i \in \V} \frac{S_i}{2 \gamma_i} - \lambda(\epsilon - \sum_{i \in V} \gamma_i D_i)$, where $D_i$ is set as in the simple method (\S \ref{sec:refine}). Optimizing gives
\begin{equation}
\gamma_i = \frac{\epsilon}{\sum_{j \in \V} \sqrt{S_j D_j}} \sqrt{ \frac{S_i}{D_i}}, \text{with } N \! \leq 2n + \frac{1}{2 \epsilon} \left( \sum_{i \in V} \sqrt{S_i D_i} \right)^2 
\end{equation}
which we term the \emph{minsum method}. Note $D_i \leq d_i W$ where $d_i$ is the degree of $X_i$, hence $\left( \sum_{i \in V} \sqrt{S_i D_i} \right)^2 \leq W \left( \sum_{i \in V} \sqrt{d_i} \right)^2$. By Cauchy-Schwartz and the handshake lemma, $\left( \sum_{i \in V} \sqrt{d_i} \right)^2 \leq n \sum_{i \in \V} d_i = 2mn$, with equality iff the $d_i$ are constant, i.e. the graph is regular. %For minimizing $N$, this method can be significantly better than the simple method if degrees are unequal, though if $D_i=d_i W$ then  it seems unlikely to be by more than a factor of two (which is achieved, for example, by a star graph).

If instead $\Pi$ is minimized, rather than $N$,  a similar argument shows that the simple method (\S \ref{sec:refine}) is optimal.

\subsubsection{Adaptive methods}\label{sec:adaptive}
The previous methods rely on one bound $D_i$ for $|\frac{\partial \F}{\partial q_i}|$ over the whole range $[A_i,1-B_i]$. However, we may increase efficiency by using local bounds to vary the mesh width across the range. A bound on the maximum magnitude of the derivative over any sub-range may be found by checking just $-f_i^L$ at the lower end and $f_i^U$ at the upper end.% (indeed $-f_i^L(q_i) > f_i^U(q_i)$ iff $q_i <  \bar{q_i}$ s.t. $f_i^U(\bar{q_i})=-f_i^L(\bar{q_i}))$.

This may be improved by using the exact integral as in \eqref{eq:intbound}. First, constant proportions $k_i>0$ should be chosen with $\sum_i k_i=1$.  
Next, the first (lowest) mesh point $\m_1^i \in \M_i$ should be set s.t. $\int_{A_i}^{\m_1^i} f_i^U(q_i) dq_i = k_i \eps$. This will ensure that $\m_1^i$ covers all points to its left in the sense that $\F[q_i=\m_1^i] - \F[q_i \in [A_i,\m_1^i]] \leq k_i \eps$ where all other variables $q_j, j \neq i$, are held constant at any values within the Bethe box. $\m_1^i$ also covers all points to its right up to what we term its \emph{reach}, i.e. the point $r_1^i$ s.t. $\int_{\m_1^i}^{r_1^i} -f_i^L(q_i) dq_i = k_i \eps$. Next, $\m_2^i$ is chosen as before, using $r_1^i$ as the left extreme rather than $A_i$, and so on, until the final mesh point is computed with reach $\geq 1-B_i$. This yields an optimal mesh for the choice of $\{k_i\}$. %[*draw this as an algorithm box? That will take more room?]

If $k_i=\frac{1}{n}$, we achieve an optimized \emph{adaptive simple} method. If $k_i=\frac{\sqrt{S_i D_i}  }{\sum_{j \in \V}\sqrt{S_j D_j} } $, we achieve an \emph{adaptive minsum} method. For many problems, this adaptive minsum method will be the most efficient.

Integrals are easily computed using \eqref{eq:integral}. To our knowledge, computing optimal points $\{\m_s^i\}$ is not possible analytically, but each may be found with high accuracy in just a few iterations using a search method% (even binary search is sufficient to return a point within $[1-\delta,1] \times$optimal within $O(1)$ iterations)
, hence total time to compute the mesh is $O(N)$, which is negligible compared to solving the subsequent MAP problem.

\section{REVISITING THE SECOND DERIVATIVE APPROACH}\label{sec:better2}

We review the second derivative approach used in \A (see \S 5 there). As here, the possible location of a global minimum $\hat{q}$ was first bounded in the Bethe box given by $\prod_{i \in \V} [A_i, 1-B_i]$. Next an upper bound $\Lambda$ was derived on the maximum possible eigenvalue of the Hessian $H$ of $\F$ anywhere within the Bethe box, where it was required that all edges be attractive. Then a mesh of constant width in every dimension was introduced s.t. the nearest mesh point $q^*$ to $\hat{q}$ was at most $\gamma$ away in each dimension. Hence the $\ell_2$ distance $\delta$ satisfies $\delta^2 \leq n \gamma^2$ and by Taylor's theorem, $F(q^*) \leq F(\hat{q}) + \frac{1}{2}\Lambda \delta^2.$ $\Lambda$ was computed by bounding the maximum magnitude of any element of $H$. Considering Theorem \ref{thm:H}, this involves separate analysis of diagonal $H_{ii}$ terms, which are positive and were bounded above by the term $b$; and edge $H_{ij}$ terms, which are negative for attractive edges, whose magnitude was bounded above by $a$. Then $\Omega$ was set as $\max(a,b)$, and $\Sigma$ as the proportion of non-zero entries in $H$. Finally, $\Lambda \leq \sqrt{\text{tr}(H^T H)} \leq \sqrt{\Sigma n^2 \Omega^2} = n \Omega \sqrt{\Sigma}$. 

\subsection{Improved bound for an attractive model}\label{sec:imp2}

We improve the upper bound for $\Lambda$ by improving the $a$ bound for attractive edges to derive $\tilde{a}$, a better upper bound on $-H_{ij}$. Essentially, a more careful analysis allows a potentially small term in the numerator and denominator to be canceled before bounding. Writing $\bar{\eta}=\min_{i \in \V} \eta_i (1-\eta_i)$, i.e. the closest that any dimension can come to 0 or 1, the result is that 
\begin{eqnarray}% \label{eq:newHij}
-H_{ij} &\leq&    \left( \frac{\al}{1+\al} \right) \Bigg/ \bar{\eta} \left( 1-\left( \frac{\al}{1+\al} \right)^2 \right) \\
&=& O(e^{W(1+\Delta/2)+T}). \nonumber
\end{eqnarray}
Thus,  $\tilde{a}=O(e^{W(1+\Delta/2)+T})$ which compares favorably to the earlier bound in \A, where $a=O(e^{W(1+\Delta)+2T})$. Recall $b=O(\Delta e^{W(1+\Delta/2)+T})$ and $\Omega=\max(a,b)$, so using the new $\tilde{a}$ bound, now $\Omega=O(\Delta e^{W(1+\Delta/2)+T})$. Details and derivation are in the supplement.

\subsection{Extending the second derivative approach to a general (non-attractive) model}\label{sec:2gen}

Using flipping arguments from \S \ref{sec:flip}, we are able to extend the method of \A to apply to general models. Interestingly, the theoretical bounds derived for $\Omega=\max(a,b)$ take exactly the same form as for the purely attractive case, except that now $-W \leq W_{ij} \leq W$, whereas previously it was required that $0 \leq W_{ij} \leq W$. Since it is a second derivative approach, the mesh size (measured by $N$, the total number of points summed over the dimensions) grows as $O(\epsilon^{-1/2})$ rather than as $O(\epsilon^{-1})$ in the new first derivative approach.  In practice, however, particularly for harder cases where $n$ and $W$ are above small values, unless $\epsilon$ is very small, the method of \S \ref{sec:1st} is much more efficient. Details and derivations are in the supplement.

\section{RESULTING MULTI-LABEL MAP}\label{sec:MAP}
After computing a sufficient mesh, it remains to solve the multi-label MAP inference problem on a MRF with the same topology as the initial model, where each $q_i$ takes values in $\M_i$. In general, this is NP-hard \cite{Shi94}.

\subsection{Tractable cases}
If it happens that all cost functions are submodular (as is always the case if the initial model is fully attractive by Theorem \ref{thm:submod}), then as already noted, it may be solved efficiently using graph cut methods, which rely on solving a max flow/min cut problem on a related graph, with worst case runtime $O(N^3)$ \cite{SchFla06,GrePorSeh89,Gol88}. Using the Boykov-Kolmogorov algorithm \cite{BoyKol04}, performance is typically much faster, sometimes approaching $O(N)$. This submodular setting is the only known class of problem which is solvable for any topology. 

Alternatively, the topological restriction of bounded tree-width allows tractable inference \cite{Pearl}. Further, under mild assumptions, this was shown to be the only restriction which will allow efficient inference for any cost functions \cite{Chand08}. We note that if the problem has bounded tree-width, then so too does the original binary pairwise model, hence exact inference (to yield the true marginals or the true partition function $Z$) on the original model is tractable, making our approximation result less interesting for this class. In contrast, although MAP inference is tractable for any attractive binary pairwise model, marginal inference and computing $Z$ are not \cite{JerSin93}. %to outerplanar graphs, then the binary MAP inference problem is tractable (ref), yet no efficient solution of a multi-label MAP problem is known.

A recent approach reducing MAP inference to identifying a maximum weight stable set in a derived weighted graph (\cite{Jeb13}, \cite{WelJeb13b}) shows promise, allowing efficient inference if the derived graph is perfect. Further, testing if this graph is perfect can be performed in polynomial time (\cite{Jeb13}, \cite{Chud05}).

\subsection{All other cases}\label{sec:othercases}

Many different methods are available, see \cite{Kap13} for a recent survey. 
Some, such as dual approaches, may provide a helpful bound even if the optimum is not found. Indeed, a LP relaxation will run in polynomial time and return an upper bound on $\log Z_B$ that may be useful. A lower bound may be found from any discrete point, and this may be improved using local search methods. Note also that BBP bounds $q_i \in [A_i,1-B_i]$ apply \emph{for all} the Bethe box, but for a particular value of $q_i$ say, then the BBP approach provides tighter bounds on each of its neighbors $j \in \N(i)$, which may be helpful for pruning the solution space.

\subsubsection{Persistent partial optimization approaches}\label{sec:pers} MQPBO \cite{MQPBO08} and Kovtun's method \cite{Kov03} are examples of this class. Both consider LP-relaxations and run in polynomial time. In our context, the output consists of ranges (which in the best case could be one point) of settings for some subset of the variables.  If any such ranges are returned, the strong persistence property ensures that \emph{any} MAP solution satisfies the ranges. Hence, these may be used to update $\{A_i,B_i\}$ bounds (padding the discretized range to the full continuous range covered by the end points if needed), compute a new, smaller, sufficient mesh and repeat until no improvement is obtained.

%\section{Discussion / Practical considerations}\label{sec:discuss}

%Show a table/graph of $N=\sum_i N_i$ (and $\prod_i N_i$?) using the different mesh methods for various values - which parameters are interesting to show? Perhaps fix others and vary $n,m,W,\epsilon$ separately?
\section{EXPERIMENTS}\label{sec:exp}

As a first step toward applying our algorithm to explore the usefulness of the global optimum of the Bethe approximation, here we consider one setting where LBP fails to converge, yet still we achieve reasonable results. 

We aim to predict transformer failures in a power network
\cite{RudinEtAl12}. Since the real data is sensitive, our experiments
use synthetic data.  Let $X_i \in \{0,1\}$ indicate if transformer $i$
has failed or not. Each transformer has a probability of failure on
its own which is represented by a singleton potential
$\theta_i$. However, when connected in a network, a transformer can
propagate its failure to nearby nodes (as in viral contagion) since
the edges in the network form associative dependencies. We assume that
homogeneous attractive pairwise potentials couple all transformers
that are connected by an edge, i.e. $W_{ij}=W \; \forall (i,j) \in \ce$. The
network topology creates a Markov random field specifying the
distribution $p(X_1,\ldots,X_n)$. Our goal is to compute the marginal
probability of failure of each transformer within the network (not
simply in isolation as in \cite{RudinEtAl12}). %Such marginals let us
%rank the transformers to determine which ones the utility company
%needs to inspect or repair first. 
Since recovering $p(X_i)$ is hard, we
estimate Bethe pseudo-marginals $q_i=q(X_i=1)$ through our
algorithm, which emerge as the $\argmin$ when optimizing the Bethe free energy.

A simulated sub-network of 55 connected transformers with average degree 2 was generated using a random preferential attachment model. Typical settings of $\theta_i=-2$ and $W=4$ were specified (using the input model specification of \S \ref{sec:input}). We attempted to run BP using the libDAI package \cite{libdai} but were unable to achieve convergence, even with multiple initial values, using various sequential or parallel settings and with damping. However, running our algorithm with $\epsilon=1$ achieved reasonable results as shown in Table \ref{tab:exp}, where true values were obtained with the junction tree algorithm.

\begin{table}[h]
\centering
\begin{tabular}{|l |c|}
\hline
$\eps=1$ PTAS for $\log Z_B$	& Error vs true value\\
\hline
%$\eps=1$ PTAS for $\log Z_B$ & \\
Mean $\ell_1$ error of single marginals & 0.003\\
Log-partition function & 0.26\\
\hline
\end{tabular}
\caption{Results on simulated power network}
\label{tab:exp}
\end{table}

General folklore has suggested that the Bethe approximation is poor when BP fails to converge, thus this initial result suggests further work, which is now feasible using our algorithm.

\section{DISCUSSION \& FUTURE WORK}\label{sec:conc}

To our knowledge, we have derived the first $\eps$-approximation algorithm for $\log Z_B$ for a general binary pairwise model. The approach is useful in practice, and much more efficient than the previous method of \cite{A}, though can take a long time to run for large, densely connected problems or when coupling is high. From experiments run, we note that the $\epsilon$ bounds appear to be close to tight since we have found models where the  optimum returned when run with $\eps=1$ is more than $0.5$ different to that for $\eps=0.1$. When applied to attractive models, we guarantee a FPTAS with no degree restriction. 

Future work includes further improving the efficiency of the mesh, considering how it should be selected to simplify the subsequent discrete optimization problem, and exploring applications. Interesting avenues include using it as a subroutine in a dual decomposition approach to optimize over a tighter relaxation of the marginal polytope, and it provides the opportunity to examine rigorously the performance of other Bethe approaches that typically run more quickly, such as LBP or CCCP \cite{Yui02}, against the true Bethe global optimum. 

\subsubsection*{Acknowledgments}

We are grateful to Kui Tang for help with coding, and to David Sontag, Kui Tang, Nicholas Ruozzi and Toma\u{z} Slivnik for helpful discussions. This material is based upon work supported by the National Science Foundation under Grant No. 1117631.

%\newpage

\bibliographystyle{mlapa}
%\bibliography{aiReferences}
\bibliography{nb-arXiv}

\clearpage
\newpage

\onecolumn

\section*{APPENDIX: SUPPLEMENTARY MATERIAL FOR APPROXIMATING THE BETHE PARTITION FUNCTION}

Here we provide further details and proofs of several of the results in the main paper, using the original numbering. 

\section*{\ref{sec:better2} \quad REVISITING THE SECOND DERIVATIVE APPROACH}

\subsection*{\ref{sec:imp2} \quad Improved bound for an attractive model}

In this section, we improve the upper bound for $\Lambda$ by improving the $a$ bound for attractive edges to derive $\tilde{a}$, an improved upper bound on $-H_{ij}$. Essentially, a more careful analysis allows a potentially small term in the numerator and denominator to be canceled before bounding. %Let $H$ be the Hessian of $\F$ with $H_{ij} = \frac{\partial^2 \F}{\partial \s \partial q_j}$.

%In \cite{A}, it was shown that $H_{ii}>0 \; \forall i \in \V$, and $H_{ij}<0$ for any attractive edge $(i,j)$. $\Omega$ is computed as $\max(a,b)$ where $a$ is an upper bound on the magnitude of $H_{ij}$ and $b$ is an upper bound on $H_{ii}$.  Here we refine the analysis to derive $a'$, an improved upper bound on $-H_{ij}$ for attractive edges. 
Using Theorem \ref{thm:H}, equation \eqref{eq:Tij} and Lemma \ref{lem:xiub},

\begin{align}\label{eq:Hij}
-H_{ij} &= (\x-\s q_j)\frac{1}{T_{ij}} \notag \\ 
	&\leq \frac{ m (1-M)\al}{1+\al} \frac{1}{m(1-M)\left[ (1-m)M-m(1-M)\left(\frac{\al}{1+\al}\right)^2 \right]} \nonumber \\
 &= \left( \frac{\al}{1+\al} \right) \frac{1}{(1-m)M-m(1-M)\left(\frac{\al}{1+\al}\right)^2}
\end{align}
where $m=\min(\s, q_j), M=\max(\s, q_j)$. Now we use the following result.

\begin{lemma}\label{lem:HijDenom}
%Anywhere in the Bethe box, if $k \in (0,1)$, then $\frac{1}{(1-m)M-m(1-M)k} = O()$. More specifically, 
For any $k \in (0,1)$, let $y=\min_{q_i \in [A_i,1-B_i], q_j \in [A_j, 1-Bj]} (1-m)M-m(1-M)k$, then %a lower bound for $(1-m)M-m(1-M)k$ may be computed as follows:
\begin{equation*}%\label{denomk}
y = \begin{cases}
B_iA_j - (1-B_i)(1-A_j)k \quad &\textnormal{if } (1-B_i) \leq A_j  \qquad \qquad \qquad \text{i range} \leq \text{j range}\\
(1-k) \min\{ A_j(1-A_j), B_i(1-B_i) \} &\tn{if } A_i \leq A_j \leq 1-B_i \leq 1-B_j  \quad \text{ranges overlap, i lower}
\\
(1-k) \min\{ A_j(1-A_j), B_j(1-B_j) \} &\tn{if } A_i \leq A_j \leq 1-B_j \leq 1-B_i  \qquad \text{j range} \subseteq \text{i range}
\\
(1-k) \min\{ A_i(1-A_i), B_i(1-B_i) \} &\tn{if } A_j \leq A_i \leq 1-B_i \leq 1-B_j  \qquad \text{i range} \subseteq \text{j range}
\\
(1-k) \min\{ A_i(1-A_i), B_j(1-B_j) \} &\tn{if } A_j \leq A_i \leq 1-B_j \leq 1-B_i \quad \text{ranges overlap, j lower}
\\
B_jA_i - (1-B_j)(1-A_i)k  &\tn{if } (1-B_j) \leq A_i  \qquad \qquad \qquad \text{j range } \leq \text{i range.}  
\end{cases}
\end{equation*}
\end{lemma}
\begin{proof}
The minimum is achieved by minimizing the larger and maximizing the smaller of $\s$ and $q_j$. The result follows for cases where their ranges are disjoint. If ranges overlap, then the minimum is achieved at some $q_i=q_j$ in the overlap, with value $q_i(1-q_i)(1-k)$, which is concave and minimized at an extreme of the overlap range.
\end{proof} 

Lemma \ref{lem:HijDenom} is useful in practice, and should be used to compute $\tilde{a}=\max_{(i,j) \in \ce}$ of the bound above. To analyze the theoretical worst case, it is straightforward to see the corollary that $y \geq (1-k) \bar{\eta}$, where $\bar{\eta}=\min_{i \in \V} \eta_i (1-\eta_i)$. This bound can be met, for example, if all ranges coincide. Hence, from \eqref{eq:Hij}, and with the reasoning for $\frac{1}{\bar{\eta}}$ from \A \S 5.3, where it is shown that $\frac{1}{\eta_i (1-\eta_i)} = O(e^{ T + \Delta W/2})$, and using $\al=e^{W_{ij}}-1$, we obtain
\begin{equation}\label{eq:newHij}
-H_{ij} \leq    \left( \frac{\al}{1+\al} \right) \Bigg/ \bar{\eta} \left( 1-\left( \frac{\al}{1+\al} \right)^2 \right) = O(e^{W(1+\Delta/2)+T}).
\end{equation}
Thus,  $\tilde{a}=O(e^{W(1+\Delta/2)+T})$ which compares favorably to the earlier bound in \A, where $a=O(e^{W(1+\Delta)+2T})$. Recall $b=O(\Delta e^{W(1+\Delta/2)+T})$ and $\Omega=\max(a,b)$, so using the new $\tilde{a}$ bound, now $\Omega=O(\Delta e^{W(1+\Delta/2)+T})$. %Note that $b$ is a bound on $H_{ii}$, which involves the sum of up to $\Delta$ terms, each of which has a similar upper bound to the new bound for $-H_{ij}$. %[Elaborate further?]

\subsection*{\ref{sec:2gen} \quad Extending the second derivative approach to a general (non-attractive) model}

Here we extend the analysis of \A by considering repulsive edges to show that for a general binary pairwise model, %, with modified $a$ and $b$ parameters, 
we can still calculate useful bounds (which turn out to be very similar to the earlier bounds for attractive models) for a sufficient mesh width.

Our main tool for dealing with a repulsive edge is to flip the variable at one end (see \S \ref{sec:flip}) to yield an attractive edge, then we can apply earlier results. We denote the flipped model parameters with a $^\prime$. For example, if just variable $X_j$ is flipped, then $q_j'=q(X_j'=1)=q(1-X_j=1)=1-q_j$. Since $\al=e^{W_{ij}}-1$ and here $W_{ij}'=-W_{ij}$, the following relationship holds if one end of an edge is flipped,% which we shall use several times:
\begin{equation}\label{eq:-al}
\frac{\al'}{1+\al'} =\frac{e^{-W_{ij}}-1}{e^{-W_{ij}}}=1-e^{W_{ij}}=-\al.
\end{equation}
Note that, for an attractive edge, $\frac{\al'}{1+\al'} \in (0,1)$, as is $-\al$ for a repulsive edge. Recall that when we flip some set of variables, by construction $\F'=\F + constant$ (see \S \ref{sec:flip}).

The Hessian terms from Theorem \ref{thm:H} still apply. Our goal is to bound the magnitude of each entry $H_{ij}$ for a general binary pairwise model, then the earlier analysis will provide the result. Whereas for a fully attractive model, we assumed a maximum edge weight $W$ with $0\leq W_{ij} \leq W$, now we assume $|W_{ij}| \leq W$.

\subsubsection*{\ref{sec:2gen}.1 \quad Edge terms} 

First consider $H_{ij}$ for an edge $(i,j) \in \ce$. If the edge is attractive, then the earlier analysis holds (it makes no difference if other edges are attractive or repulsive).  If it is repulsive, then $H_{ij}>0$. Consider a model where just $X_j$ is flipped. $H_{ij} = \frac{ \partial^2 \F} {\partial \s \partial q_j} = -\frac{ \partial^2 \F'} {\partial \s' \partial q_j'} = -H_{ij}'$. Hence using \eqref{eq:Hij} and \eqref{eq:-al}, in practice an upper bound may be computed from Lemma \ref{lem:HijDenom} using $k=-\al$ and $A_j'=B_j, B_j'=A_j$. The theoretical bound for an attractive edge from \eqref{eq:newHij} becomes $H_{ij} \leq \frac{-\al}{\bar{\eta}(1-\al^2)}$. As we should expect from the attractive case, the following result holds.
\begin{lemma}\label{lem:al2}
For a repulsive edge, $\frac{1}{1-\al^2}=O(e^{-W_{ij}})$. 
\end{lemma}
\vspace{-.5cm} 
\begin{proof}
Let $u=-W_{ij}$, then $\al=e^{-u}-1$ and $\frac{1}{1-\al^2} = 
\frac{1}{(1-\al)(1+\al)}
=\frac{1}{e^{-u}(2-e^{-u})} 
=O(e^u)$.
\end{proof}
Hence, noting that we may flip any neighbors $j$ of $i$ which are adjacent via repulsive edges to obtain $\frac{1}{\eta_i (1-\eta_i)} = O(e^{T+\Delta W/2})$ as before, where now $W=\max_{(i,j)\in \ce} |W_{ij}|$, we see that for our new second derivative method, just as in the fully attractive case, $\tilde{a}=O(e^{W(1+\Delta/2)+T})$.

For comparison interest, we also show how the earlier, worse bound for an attractive edge given in \A may similarly be combined with flipping to provide a worse upper bound for $H_{ij}$ when $(i,j)$ is repulsive. See \A \S 5.2: considering the proof of Lemma 10 and using \eqref{eq:-al} from this paper, we see that for a repulsive edge, the $K_{ij}$ minimum bound for $T_{ij}$ becomes $K_{ij}=\eta_i \eta_j (1-\eta_i)(1-\eta_j)(1-\al^2)$; then from \A Theorem 11, the equivalent bound is $H_{ij} \leq \frac{ -\al} {4 K_{ij}}$ which gives  $a=O(e^{W(1+\Delta)+2T})$ as it was for the fully attractive case.

\smallskip
We provide a further interesting result, deriving a lower bound for $\x$ for a repulsive edge. 
%Consider Theorem \ref{thm:2nderiv} and equation \eqref{eq:Tij}. For attractive edges, $\x \geq \s q_j$ (Lemma \ref{lem:newst}) and an upper bound is provided by Lemma \ref{lem:xiub}. For a repulsive edge, where $\x \leq \s q_j$, our approach is to flip one of the variables (see \S \ref{sec:flip}) to yield an associative edge, on which we can use the earlier results, and thence derive an upper bound on $\s q_j - \x$. Recall that $\al=e^{W_{ij}}-1 \in (-1,0)$ for a repulsive edge.
\begin{lemma}[Lower bound for $\x$ for a repulsive edge, analogue of Lemma \ref{lem:xiub}]\label{lem:xilb}For any repulsive edge $(i,j)$,\\ $\s q_j - \x \leq -\al p_{ij}$ where $p_{ij} = \min\{q_i q_j, (1-q_i)(1-q_j)\}$.
\end{lemma}
\begin{proof}
Consider a model where just variable $X_j$ is flipped, and let all new quantities be designated by the symbol $^\prime$%', e.g. $q_j'=q(X_j'=1)=q(1-X_j=1)=1-q_j$
. Consider the joint pseudo-marginal \eqref{eq:mu}. In the new model the columns are switched since $\mu_{ij}'(a,b)=q(X_i'=a,X_j'=b)
%=q(X_i=a, 1-X_j=b)
=q(X_i=a,X_j=1-b)=\mu_{ij}(a,1-b)$, hence
\begin{align}\label{eq:mup}
\mu_{ij}' &=\begin{pmatrix} 1 + \x' -\s' -q_j' & q_j'-\x' \\ \s' - \x' & \x' \end{pmatrix}  
	\:\: = \begin{pmatrix} q_j-\x & 1 + \x -\s -q_j  \\ \x & \s - \x \end{pmatrix}.
\end{align}
Applying Lemma \ref{lem:xiub} to the new model, $\x'-\s' q_j' \leq \frac{\al'}{1+\al'}m'(1-M')$. Substituting in $\x'=\s-\x$ from \eqref{eq:mup} and using \eqref{eq:-al},  %noting that $W_{ij}'=-W_{ij}$ so $\frac{\al'}{1+\al'} =\frac{e^{-W_{ij}}-1}{e^{-W_{ij}}}=1-e^W_{ij}=-\al$, 
we have $(\s - \x) - \s (1-q_j) \leq -\al m'(1-M')$. %, hence $\s q_j -\x \leq m'(1-M')$
Since $m'=\min\{\s, 1-q_j\}$ and $M'=\max\{\s, 1-q_j\}$, noting $\s \leq 1-q_j \Leftrightarrow \s+q_j \leq 1 \Leftrightarrow \s q_j \leq (1-\s)(1-q_j)$, the result follows.
\end{proof}

Hence for a repulsive edge $(i,j)$, using \eqref{eq:Tij}, we have
\begin{equation*}\label{eq:TijR}
T_{ij} = \s q_j(1-\s)(1-q_j) -(\x-\s q_j)^2 \geq p_{ij}P_{ij} - \al^2 p_{ij}^2,
\end{equation*}
where  $P_{ij}=\max\{q_i q_j, (1-q_i)(1-q_j)\}$.

\subsubsection*{\ref{sec:2gen}.2 \quad Diagonal terms} 
Consider the $H_{ii}$ terms from Theorem \ref{thm:H}, which is true for a general model. If all neighbors of $X_i$ are adjacent via attractive edges, then, as in \A Theorem 11,
%\begin{equation}\label{eq:Hii}
$H_{ii} \leq \frac{1}{\eta_i(1-\eta_i)} \left( 1-d_i + \sum_{j \in \N(i)}   \frac{1} {1-\left(\frac{\al}{1+\al} \right)^2 } \right)$.

%$H_{ii} \leq \frac{1-d_i}{\eta_i(1-\eta_i)} + \sum_{j \in \N(i)} \frac{1}{ q_i (1-q_i)\Big[1-\Big(\frac{\al}{1+\al} \Big)^2 \Big] }$.
%\end{equation}
If any neighbors are connected to $X_i$ by a repulsive edge, then consider a new model where those neighbors are flipped, so now all edges incident to $X_i$ are attractive, and designate the new model parameters with a $^\prime$. As before, observe $\F=\F'+constant$, hence $H_{ii}=\frac{\partial^2 \F}{\partial q_i^2} = \frac{\partial^2 \F'}{\partial q_i'^2} = H_{ii}'$. Using \eqref{eq:-al} we obtain that for a general model,
\begin{equation}\label{eq:genHii}
H_{ii} \leq \frac{1}{\eta_i(1-\eta_i)} \left( 1-d_i + \sum_{j \in \N(i): W_{ij}>0} \frac{1}{1-\left(\frac{\al}{1+\al} \right)^2 } +\sum_{j \in \N(i): W_{ij}<0} \frac{1}{ 1-\al^2} \right).
%H_{ii} \leq \frac{1-d_i}{\eta_i(1-\eta_i)} + \sum_{j \in \N(i): W_{ij}>0} \frac{1}{ q_i (1-q_i)\Big[1-\Big(\frac{\al}{1+\al} \Big)^2 \Big] } +\sum_{j \in \N(i): W_{ij}<0} \frac{1}{ q_i (1-q_i) [1-\al^2]} .
\end{equation}
Similarly to the analysis in \S \ref{sec:2gen}.1, using Lemma \ref{lem:al2} gives that for a general model, $b=\max_{i \in \V} H_{ii} = O(\Delta e^{W(1+\Delta/2)+T})$, just as for a fully attractive model, where now $W=\max |W_{ij}|$.

\end{document}